%% file: Paper.tex
\pdfoutput=1
\documentclass[letterpaper]{article}
\usepackage{uai2019}
\usepackage[margin=1in]{geometry}
\usepackage{times}
\usepackage[square,numbers]{natbib}

\usepackage[utf8]{inputenc} 
\usepackage[T1]{fontenc}    
\usepackage{hyperref}       
\usepackage{url}            
\usepackage{booktabs}       
\usepackage{amsfonts}       
\usepackage{nicefrac}       
\usepackage{microtype}      

\include{mydef}

\usepackage{tabu}
\usepackage{paralist}
\usepackage{color}
\usepackage[ruled]{algorithm2e}
\usepackage[noend]{algorithmic}
\usepackage{amsbsy}
\usepackage{amsmath}
\usepackage{bm}
\usepackage{dsfont}
\usepackage[all]{xy}

\usepackage{graphicx}
\usepackage{epstopdf}
\usepackage{epsfig}
\graphicspath{{./figures/}}

\usepackage[usenames,dvipsnames]{xcolor}
\usepackage[backgroundcolor=white]{todonotes}

\usepackage[capitalize,noabbrev]{cleveref}

\usepackage[skip=0pt]{caption}

\newtheorem{assumption}{Assumption}
\newtheorem{mydefinition}{Definition}

\newtheorem{theorem}[mydefinition]{Theorem}
\newtheorem{lemma}[mydefinition]{Lemma}

\newcommand{\cD}{\mathcal{D}}
\newcommand{\cE}{\mathcal{E}}
\newcommand{\ccE}{\overline{\cE}}
\newcommand{\cF}{\mathcal{F}}
\newcommand{\cP}{\mathcal{P}}
\newcommand{\cR}{\mathcal{R}}
\newcommand{\cV}{\mathcal{V}}

\newcommand{\Et}[1]{\mathbb{E}_t \left[#1\right]}
\newcommand{\Etp}[1]{\mathbb{E}_{t - 1} \left[#1\right]}

\newcommand{\abs}[1]{\left|#1\right|}

\newcommand{\condE}[2]{\mathbb{E} \left[#1 \,\middle|\, #2\right]}
\newcommand{\condEsub}[3]{\mathbb{E}_{#3} \! \left[#1 \,\middle|\, #2\right]}

\newcommand{\E}[1]{\mathbb{E} \left[#1\right]}

\newcommand{\floors}[1]{\left\lfloor#1\right\rfloor}
\newcommand{\I}[1]{\mathds{1} \! \left\{#1\right\}}

\newcommand{\rnd}[1]{\bm{#1}}
\newcommand{\set}[1]{\left\{#1\right\}}

\mathchardef\mhyphen="2D

\newcommand{\baseline}{{\tt Baseline}}
\newcommand{\batchrank}{{\tt BatchRank}}
\newcommand{\bubblerank}{{\tt BubbleRank}}

\newcommand{\cascadeklucb}{{\tt CascadeKL\mhyphen UCB}}

\newcommand{\toprank}{{\tt TopRank}}

\usepackage{enumitem}
\setlist[description]{leftmargin=\parindent,labelindent=\parindent}

\title{BubbleRank: Safe Online Learning to Re-Rank via Implicit Click Feedback}

\author{
{\bf Chang Li} \\
University of Amsterdam \\
\texttt{c.li@uva.nl} 
\And
{\bf Branislav Kveton\thanks{This work was done while the author was at Adobe Research.}} \\
Google Research \\
\texttt{bkveton@google.com} 
\And
{\bf Tor Lattimore} \\
Deepmind \\
\texttt{tor.lattimore@gmail.com} 
\And
{\bf Ilya Markov} \\
University of Amsterdam \\
\texttt{i.markov@uva.nl} 
\And
{\bf Maarten de Rijke} \\
University of Amsterdam \\
\texttt{derijke@uva.nl} 
\And
{\bf Csaba Szepesv\'ari} \\
Deepmind \\
\texttt{szepesva@cs.ualberta.ca} 
\And
{\bf Masrour Zoghi} \\
Google Research \\
\texttt{masrour@zoghi.org}
}
\author{\textbf{Chang Li}$^1$\quad 
\textbf{Branislav Kveton}$^2$\quad 
{\bf Tor Lattimore}$^3$\quad 
{\bf Ilya Markov}$^1$\\
{\bf Maarten de Rijke}$^1$ \quad 
{\bf Csaba Szepesv\'ari}$^{3, 4}$\quad
{\bf Masrour Zoghi}$^2$\\
$^1$University of Amsterdam \quad 
$^2$Google Research \quad
$^3$DeepMind\quad
$^4$University of Alberta\\
\texttt{c.li@uva.nl}, \texttt{bkveton@google.com}, \texttt{tor.lattimore@gmail.com}, \texttt{i.markov@uva.nl},\\ \texttt{derijke@uva.nl},  \texttt{szepesva@cs.ualberta.ca}, \texttt{masrour@zoghi.org}}

\begin{document}
\maketitle

\begin{abstract}
In this paper, we study the problem of safe online learning to re-rank, where user feedback is used to improve the quality of displayed lists. 
Learning to rank has traditionally been studied in two settings. In the offline setting, rankers are typically learned from relevance labels created by judges. 
This approach has generally become standard in industrial applications of ranking, such as search. 
However, this approach lacks exploration and thus is limited by the information content of the offline training data. 
In the online setting, an algorithm can experiment with lists and learn from feedback on them in a sequential fashion. 
Bandit algorithms are well-suited for this setting but they tend to learn user preferences from scratch, which results in a high initial cost of exploration. 
This poses an additional challenge of \emph{safe} exploration in ranked lists. 
We propose $\bubblerank$, a bandit algorithm for safe re-ranking that combines the strengths of both the offline and online settings. 
The algorithm starts with an initial base list and improves it online by gradually exchanging higher-ranked less attractive items for lower-ranked more attractive items. 
We prove an upper bound on the $n$-step regret of $\bubblerank$ that degrades gracefully with the quality of the initial base list. 
Our theoretical findings are supported by extensive experiments on a large-scale real-world click dataset.
\end{abstract}

\input{Introduction}

\input{Background}

\input{Algorithm}

\input{Analysis}

\input{Experiments}

\input{RelatedWork}

\input{Conclusions}

\subsection*{ACKNOWLEDGMENTS}
We thank our reviewers for helpful feedback and suggestions.
This research was supported by
Ahold Delhaize,
the Association of Universities in the Netherlands (VSNU),
the Innovation Center for Artificial Intelligence (ICAI),
and
the Netherlands Organisation for Scientific Research (NWO)
under pro\-ject nr.\ 612.\-001.\-551.
All content represents the opinion of the authors, which is not necessarily shared or endorsed by their respective employers and/or sponsors.

\bibliographystyle{abbrvnat}
\bibliography{Bibliography} 

\input{Appendix}

\end{document}

%% file: mydef.tex
\usepackage{amsmath,amsthm,amssymb}

\usepackage{color}
\definecolor{Blue}{rgb}{0.9,0.3,0.3}
\usepackage{cancel}

\newcommand{\squishlist}{
   \begin{list}{$\bullet$}
    { \setlength{\itemsep}{0pt}      \setlength{\parsep}{3pt}
      \setlength{\topsep}{3pt}       \setlength{\partopsep}{0pt}
      \setlength{\leftmargin}{1.5em} \setlength{\labelwidth}{1em}
      \setlength{\labelsep}{0.5em} } }

\newcommand{\squishlisttwo}{
   \begin{list}{$\bullet$}
    { \setlength{\itemsep}{0pt}    \setlength{\parsep}{0pt}
      \setlength{\topsep}{0pt}     \setlength{\partopsep}{0pt}
      \setlength{\leftmargin}{2em} \setlength{\labelwidth}{1.5em}
      \setlength{\labelsep}{0.5em} } }

\newcommand{\squishend}{
    \end{list}  }

\newcommand{\be}{\begin{equation}}
\newcommand{\ee}{\end{equation}}
\newcommand{\bea}{\begin{eqnarray}}
\newcommand{\eea}{\end{eqnarray}}
\newcommand{\beaa}{\begin{eqnarray*}}
\newcommand{\eeaa}{\end{eqnarray*}}

\DeclareMathAlphabet{\mathpzc}{OT1}{pzc}{m}{n}



%% file: Introduction.tex

\section{INTRODUCTION}
\label{sec:introduction}

Learning to rank (LTR) is an important problem in many application domains, such as information retrieval, ad placement, and recommender systems \cite{liu2009learning}. More generally, LTR arises in any situation where multiple items, such as web pages, are presented to users. It is particularly relevant when the diversity of users makes it hard to decide which item should be presented to a specific user~\cite{radlinski08learning,Yue2011Linear}.

A traditional approach to LTR is offline learning of rankers from either relevance labels created by judges \cite{qin2010letor} or user interactions \cite{joachims2002optimizing,mcmahan2013ad}. Recent experimental results \cite{clicklambda} shows that such rankers, even in a highly-optimized search engine, can be improved by online LTR with exploration. Exploration is the key component in multi-armed bandit algorithms \cite{auer02finitetime}. Many such algorithms have been proposed recently for online LTR in specific user-behavior models \cite{katariya16dcm,cascadingbandit,lagree16multipleplay}, the so-called \emph{click models} \cite{chuklin2015click}. 
Compared to earlier online LTR algorithms \cite{radlinski08learning}, these click model-based algorithms gain in statistical efficiency while giving up on generality. Empirical results indicate that click model-based algorithms are likely to be beneficial in practice.

Yet, existing algorithms for online LTR in click models are impractical for at least three reasons. First, an actual model of user behavior is typically unknown. 
This problem was initially addressed by \citet{batchrank}. They showed that the list of items in the descending order of relevance is optimal in several click models and proposed $\batchrank$ for learning it. 
Then \citet{lattimore2018toprank} built upon this work and proposed $\toprank$, which is the state-of-the-art online LTR algorithm. 
Second, these algorithms lack safety constraints and explore aggressively by placing potentially irrelevant items at high positions, which may significantly degrade user experience~\cite{Wang:2018}. A third and related problem is that the algorithms are not well suited for so-called \emph{warm start} scnearios~\cite{Vorobev2015Gathering}, where the offline-trained production ranker already generates a good list, which only needs to be safely improved. 
Warm-starting an online LTR algorithm is challenging since existing posterior sampling algorithms, such as Thompson sampling~\cite{thompson-1933}, require item-level priors while only list-level priors are available practically. 

We make the following contributions. First, motivated by the exploration scheme of \citet{Radlinski:2006}, we propose a bandit algorithm for online LTR that addresses all three issues mentioned above. The proposed algorithm gradually improves upon an initial base list by exchanging  higher-ranked less attractive items for lower-ranked more attractive items. The algorithm resembles bubble sort \cite{DBLP:books/daglib/0023376}, and therefore we call it $\bubblerank$. Second, we prove an upper bound on the $n$-step regret of $\bubblerank$. The bound reflects the behavior of $\bubblerank$: worse initial base lists lead to a higher regret.
Third, we define our safety constraint, which is based on incorrectly-ordered item pairs in the ranked list, and prove that $\bubblerank$ never violates this constraint with a high probability. Finally, we evaluate $\bubblerank$ extensively on a large-scale real-world click dataset.

%% file: Background.tex

\section{BACKGROUND}
\label{sec:background}

This section introduces our online learning problem. We first review click models \cite{chuklin2015click} and then introduce a stochastic click bandit~\cite{batchrank}, a learning to rank framework for multiple click models.

The following notation is used in the rest of the paper. We denote $\set{1, \dots, n}$ by $[n]$. For any sets $A$ and $B$, we denote by $A^B$ the set of all vectors whose entries are indexed by $B$ and take values from $A$. We use boldface letters to denote random variables.

\subsection{CLICK MODELS}
\label{sec:click models}

A \emph{click model} is a model of how a user clicks on a list of documents. We refer to the documents as \emph{items} and denote the universe of all items by $\cD = [L]$. The user is presented a \emph{ranked list}, an ordered list of $K$ documents out of $L$. We denote this list by $\cR \in \Pi_K(\cD)$, where $\Pi_K(\cD)$ is the set of all $K$-tuples with distinct items from $\cD$. We denote by $\cR(k)$ the item at position $k$ in $\cR$; and by $\cR^{-1}(i)$ the position of item $i$ in $\cR$, if item $i$ is in $\cR$.

Many click models are parameterized by \emph{item-dependent attraction probabilities} $\alpha \in [0, 1]^L$, where $\alpha(i)$ is the \emph{attraction probability} of item $i$. We discuss the two most fundamental click models below.

In the \emph{cascade model} (CM)~\cite{craswell08experimental}, the user scans list $\cR$ from the first item $\cR(1)$ to the last $\cR(K)$. If item $\cR(k)$ is \emph{attractive}, the user \emph{clicks} on it and does not examine the remaining items. If item $\cR(k)$ is not attractive, the user \emph{examines} item $\cR(k + 1)$. The first item $\cR(1)$ is examined with probability one. Therefore, the expected number of clicks is equal to the probability of clicking on any item, and is $r(\cR) = \sum_{k = 1}^K \chi(\cR, k) \alpha(\cR(k))$, where $\chi(\cR, k) = \prod_{i = 1}^{k - 1} (1 - \alpha(\cR(i)))$ is the examination probability of position $k$ in list $\cR$.

In the \emph{position-based model} (PBM)~\cite{richardson07predicting}, the probability of clicking on an item depends on both its identity and position. Therefore, in addition to item-dependent attraction probabilities $\alpha$, the PBM is parameterized by $K$ \emph{position-dependent examination probabilities} $\chi \in [0, 1]^K$, where $\chi(k)$ is the examination probability of position $k$. The user interacts with list $\cR$ as follows. The user \emph{examines} position $k \in [K]$ with probability $\chi(k)$ and then \emph{clicks} on item $\cR(k)$ at that position with probability $\alpha(\cR(k))$. Therefore, the expected number of clicks on list $\cR$ is $r(\cR) = \sum_{k = 1}^K \chi(k) \alpha(\cR(k))$.

CM and PBM are similar models, because the probability of clicking factors into item and position dependent factors. Therefore, both in the CM and PBM, under the assumption that $\chi(1) \geq \dots \geq \chi(K)$, the expected number of clicks is maximized by listing the $K$ most attractive items in descending order of their attraction. More precisely, the most clicked list is
\begin{align}
  \cR^\ast = (1, \dots, K)
  \label{eq:optimal list}
\end{align}
when $\alpha(1) \geq \dots \geq \alpha(L)$. Therefore, perhaps not surprisingly, the problem of learning the optimal list in both models can be viewed as the same problem, a stochastic click bandit \cite{batchrank}.

\subsection{STOCHASTIC CLICK BANDIT}
\label{sec:stochastic click bandit}

An instance of a \emph{stochastic click bandit} \cite{batchrank} is a tuple $(K, L, P_\alpha, P_\chi)$, where $K \leq L$ is the number of positions, $L$ is the number of items, $P_\alpha$ is a distribution over binary attraction vectors $\set{0, 1}^L$, and $P_\chi$ is a distribution over binary examination matrices $\set{0, 1}^{\Pi_K(\cD) \times K}$.

The learning agent interacts with the stochastic click bandit as follows. At time $t$, it chooses a list $\rnd{\cR}_t \in \Pi_K(\cD)$, which depends on its history up to time $t$, and then observes \emph{clicks} $\rnd{c}_t \in \set{0, 1}^K$ on all positions in $\rnd{\cR}_t$. A position is clicked if and only if it is examined and the item at that position is attractive. More specifically, for any $k \in [K]$,
\begin{align}
  \rnd{c}_t(k) = \rnd{X}_t(\rnd{\cR}_t, k) \rnd{A}_t(\rnd{\cR}_t(k)),
  \label{eq:click}
\end{align}
where $\rnd{X}_t \in \set{0, 1}^{\Pi_K(\cD) \times K}$ and $\rnd{X}_t(\cR, k)$ is the \emph{examination indicator} of position $k$ in list $\cR \in \Pi_K(\cD)$ at time $t$; and $\rnd{A}_t \in \set{0, 1}^L$ and $\rnd{A}_t(i)$ is the \emph{attraction indicator} of item $i$ at time $t$. Both $\rnd{A}_t$ and $\rnd{X}_t$ are stochastic and drawn i.i.d. from $P_\alpha \otimes P_\chi$.

The key assumption that allows learning in this model is that the attraction of any item is independent of the examination of its position. In particular, for any list $\cR \in \Pi_K(\cD)$ and position $k \in [K]$,
\begin{align}
  \condE{\rnd{c}_t(k)}{\rnd{\cR}_t = \cR} =
  \chi(\cR, k) \alpha(\cR(k)),
\end{align}
where $\alpha = \E{\rnd{A}_t}$ and $\alpha(i)$ is the \emph{attraction probability} of item $i$; and $\chi = \E{\rnd{X}_t}$ and $\chi(\cR, k)$ is the \emph{examination probability} of position $k$ in $\cR$. Note that the above independence assumption is in expectation only. We do not require that the clicks are independent of the position or other displayed items.

The \emph{expected reward} at time $t$ is the expected number of clicks at time $t$. Based on our independence assumption, $\sum_{k = 1}^K \E{\rnd{c}_t(k)} = r(\rnd{\cR}_t, \alpha, \chi)$, where $r(\cR, A, X) = \sum_{k = 1}^K X(\cR, k) A(\cR(k))$ for any $\cR \in \Pi_K(\cD)$, $A \in [0, 1]^L$, and $X \in [0, 1]^{\Pi_K(\cD) \times K}$. The learning agent maximizes the expected number of clicks in $n$ steps. This problem can be equivalently viewed as minimizing the \emph{expected cumulative regret} in $n$ steps, which we define as
\begin{align}
  \mbox{}\hspace*{-2mm}R(n) {=}
  \sum_{t = 1}^n \E{\max_{\cR \in \Pi_K(\cD)} r(\cR, \alpha, \chi) - r(\rnd{\cR}_t, \alpha, \chi)}
  \hspace*{-.75mm}\mbox{}.
  \label{eq:regret}
\end{align}

%% file: Algorithm.tex

\section{ONLINE LEARNING TO RE-RANK}
\label{sec:learning to re-rank}

Multi-stage ranking is widely used in production ranking systems \cite{chen2017efficient,karmakersantu2017application,liu2017cascade}, with the re-ranking stage at the very end \cite{chen2017efficient}. In the re-ranking stage, a relatively small number of items, typically $10$--$20$, are re-ranked. One reason for re-ranking is that offline rankers are typically trained to minimize the average loss across a large number of queries. Therefore, they perform well on very frequent queries and poorly on infrequent queries. On moderately frequent queries, the so-called \emph{torso queries}, their performance varies. As torso queries are sufficiently frequent, an online algorithm can be used to re-rank so as to optimize their value, such as the number of clicks \cite{clicklambda}.

We propose an online algorithm that addresses the above problem and adaptively re-ranks a list of items generated by a production ranker with the goal of placing more attractive items at higher positions. We study a non-contextual variant of the problem, where we re-rank a small number of items in a single query. Generalization across queries and items is an interesting direction for future work. We follow the setting in \cref{sec:stochastic click bandit}, except that $\cD = [K]$. Despite these simplifying assumptions, our learning problem remains a challenge. In particular, the attraction of items is only observed through clicks in \eqref{eq:click}, which are affected by other items in the list.

\subsection{ALGORITHM}
\label{sec:algorithm}

\begin{algorithm}[t]
	\caption{$\bubblerank$}
	\label{alg:bubble rank}
	\begin{algorithmic}[1]
		\STATE \textbf{Input:} initial list $\cR_0$ over $[K]$ 
		\vspace{0.075in}
		\STATE $\forall i, j \in [K]: \rnd{s}_0(i, j) \gets 0, \ \rnd{n}_0(i, j) \gets 0$
		\STATE $\bar{\rnd{\cR}}_1 \gets \cR_0$ \label{alg:3}
		\FOR{$t = 1, \dots, n$}
		\STATE $h \gets t \ \mathrm{mod} \ 2$
		\vspace{0.075in}
		\STATE $\rnd{\cR}_t \gets \bar{\rnd{\cR}}_t$ \label{alg:6}
		\FOR{$k = 1, \dots, \floors{(K - h) / 2}$}
		\STATE $i \gets \rnd{\cR}_t(2 k - 1 + h), \ j \gets \rnd{\cR}_t(2 k + h)$
		\IF{$\rnd{s}_{t - 1}(i, j) \leq 2 \sqrt{\rnd{n}_{t - 1}(i, j) \log(1 / \delta)}$} \label{alg:10}
		\STATE Randomly exchange items $\rnd{\cR}_t(2 k - 1 + h)$ and $\rnd{\cR}_t(2 k + h)$ in list $\rnd{\cR}_t$
		\ENDIF \label{alg:11}
		\ENDFOR
		\vspace{0.075in}
		\STATE Display list $\rnd{\cR}_t$ and observe clicks $\rnd{c}_t \in \{0, 1\}^K$ \label{alg:13}
		\vspace{0.075in}
		\STATE $\rnd{s}_t \gets \rnd{s}_{t - 1}, \ \rnd{n}_t \gets \rnd{n}_{t - 1}$ \label{alg:15}
		\FOR{$k = 1, \dots, \floors{(K - h) / 2}$}
		\STATE $i \gets \rnd{\cR}_t(2 k - 1 + h), \ j \gets \rnd{\cR}_t(2 k + h)$
		\IF{$\abs{\rnd{c}_t(2 k - 1 + h) - \rnd{c}_t(2 k + h)} = 1$} \label{alg:18}
		\STATE $\rnd{s}_t(i, j) \gets \rnd{s}_t(i, j) + \rnd{c}_t(2 k - 1 + h) - \rnd{c}_t(2 k + h)$
		\STATE $\rnd{n}_t(i, j) \gets \rnd{n}_t(i, j) + 1$
		\STATE $\rnd{s}_t(j, i) \gets \rnd{s}_t(j, i) + \rnd{c}_t(2 k + h) - \rnd{c}_t(2 k - 1 + h)$
		\STATE $\rnd{n}_t(j, i) \gets \rnd{n}_t(j, i) + 1$ \label{alg:22}
		\ENDIF
		\ENDFOR
		\vspace{0.075in}
		\STATE $\bar{\rnd{\cR}}_{t + 1} \gets \bar{\rnd{\cR}}_t$ \label{alg:24}
		\FOR{$k = 1, \dots, K - 1$}
		\STATE $i \gets \bar{\rnd{\cR}}_{t + 1}(k), \ j \gets \bar{\rnd{\cR}}_{t + 1}(k + 1)$
		\IF{$\rnd{s}_t(j, i) > 2 \sqrt{\rnd{n}_t(j, i) \log(1 / \delta)}$} \label{alg:27}
		\STATE Exchange items $\bar{\rnd{\cR}}_{t + 1}(k)$ and $\bar{\rnd{\cR}}_{t + 1}(k + 1)$ in list $\bar{\rnd{\cR}}_{t + 1}$ \label{alg:28}
		\ENDIF
		\ENDFOR
		\ENDFOR
	\end{algorithmic}
\end{algorithm}

Our algorithm is presented in \cref{alg:bubble rank}. The algorithm gradually improves upon an \emph{initial base list} $\cR_0$ by ``bubbling up'' more attractive items. Therefore, we refer to it as $\bubblerank$. $\bubblerank$ determines more attractive items by randomly exchanging neighboring items. If the lower-ranked item is found to be more attractive, the items are permanently exchanged and never randomly exchanged again. If the lower-ranked item is found to be less attractive, the items are never randomly exchanged again. We describe $\bubblerank$ in detail below.

$\bubblerank$ maintains a \emph{base list} $\bar{\rnd{\cR}}_t$ at each time $t$. From the viewpoint of $\bubblerank$, this is the best list at time $t$. The list is initialized by the initial base list $\cR_0$ (line~\ref{alg:3}). At time $t$, $\bubblerank$  permutes $\bar{\rnd{\cR}}_t$ into a \emph{displayed list} $\rnd{\cR}_t$ (lines~\ref{alg:6}--\ref{alg:11}). Two kinds of permutations are employed. If $t$ is odd and so $h = 0$, the items at positions $1$ and $2$, $3$ and $4$, and so on, are randomly exchanged. If $t$ is even and so $h = 1$, the items at positions $2$ and $3$, $4$ and $5$, and so on are randomly exchanged. The items are exchanged only if $\bubblerank$ is uncertain regarding which item is more attractive (line~\ref{alg:10}).

The list $\rnd{\cR}_t$ is displayed and $\bubblerank$ gets feedback (line~\ref{alg:13}). Then it updates its statistics (lines~\ref{alg:15}--\ref{alg:22}). For any exchanged items $i$ and $j$, if item $i$ is clicked and item $j$ is not, the belief that $i$ is more attractive than $j$, $\rnd{s}_t(i, j)$, increases; and the belief that $j$ is more attractive than $i$, $\rnd{s}_t(j, i)$, decreases. The number of observations, $\rnd{n}_t(i, j)$ and $\rnd{n}_t(j, i)$, increases. These statistics are updated only if one of the items is clicked (line~\ref{alg:18}), not both.

At the end of time $t$, the base list $\bar{\rnd{\cR}}_t$ is improved (lines~\ref{alg:24}--\ref{alg:28}). More specifically, if any lower-ranked item $j$ is found to be more attractive than its higher-ranked neighbor $i$ (line~\ref{alg:27}), the items are permanently exchanged in the next base list $\bar{\rnd{\cR}}_{t + 1}$.

A notable property of $\bubblerank$ is that it explores safely, since any item in the displayed list $\rnd{\cR}_t$ is at most one position away from its position in the base list $\bar{\rnd{\cR}}_t$. Moreover, any base list improves upon the initial base list $\cR_0$, because it is obtained by bubbling up more attractive items with a high confidence. We make this notion of safety more precise in \cref{sec:safety}.

%% file: Analysis.tex

\section{THEORETICAL ANALYSIS}
\label{sec:analysis}

In this section, we provide theoretical guarantees on the performance of $\bubblerank$, by bounding the $n$-step regret in \eqref{eq:regret}.

The content is organized as follows. In \cref{sec:regret bound}, we present our upper bound on the $n$-step regret of $\bubblerank$, together with our assumptions. In \cref{sec:safety}, we prove that $\bubblerank$ is safe. In \cref{sec:discussion}, we discuss our theoretical results. The regret bound is proved in \cref{sec:upper bound proof}. Our technical lemmas are stated and proved in \cref{sec:technical lemmas}.

\subsection{REGRET BOUND}
\label{sec:regret bound}

Before we present our result, we introduce our assumptions\footnote{Our assumptions are slightly weaker than those of  \citet{batchrank}. For instance, Assumption~\ref{ass:order-free examination} is on the probability of examination. \citet{batchrank} make this assumption on the realization of examination.} and complexity metrics.

\begin{assumption}
\label{sec:assumptions} For any lists $\cR, \cR' \in \Pi_K(\cD)$ and positions $k, \ell \in [K]$ such that $k < \ell$:
\begin{enumerate}[label=A\arabic*.,align=left,leftmargin=*,ref=A\arabic*]
\item \label{ass:optimal list} $r(\cR, \alpha, \chi) \leq r(\cR^\ast, \alpha, \chi)$, where $\cR^\ast$ is defined in \eqref{eq:optimal list};
\item \label{ass:order-free examination} $\set{\cR(1), \dots, \cR(k - 1)} = \set{\cR'(1), \dots, \cR'(k - 1)}$ $\implies \chi(\cR, k) = \chi(\cR', k)$;
\item \label{ass:decreasing examination} $\chi(\cR, k) \geq \chi(\cR, \ell)$;
\item \label{ass:examination scaling} If $\cR$ and $\cR'$ differ only in that the items at positions $k$ and $\ell$ are exchanged, then $\alpha(\cR(k)) \leq \alpha(\cR(\ell)) \iff \chi(\cR, \ell) \geq \chi(\cR', \ell)$; and
\item \label{ass:lowest examination} $\chi(\cR, k) \geq \chi(\cR^\ast, k)$.
\end{enumerate}
\end{assumption}

The above assumptions hold in the CM. In the PBM, they hold when the examination probability decreases with the position.

Our assumptions can be interpreted as follows. Assumption~\ref{ass:optimal list} says that the list of items in the descending order of attraction probabilities is optimal. Assumption~\ref{ass:order-free examination} says that the examination probability of any position depends only on the identities of higher-ranked items. Assumption~\ref{ass:decreasing examination} says that a higher position is at least as examined as a lower position. Assumption~\ref{ass:examination scaling} says that a higher-ranked item is less attractive if and only if it increases the examination of a lower position. Assumption~\ref{ass:lowest examination} says that any position is examined the least in the optimal list.

To simplify our exposition, we assume that $\alpha(1) > \dots > \alpha(K) > 0$. Let $\chi_{\max} = \chi(\cR^\ast, 1)$ denote the \emph{maximum examination probability}, $\chi_{\min} = \chi(\cR^\ast, K)$ denote the \emph{minimum examination probability}, and
\begin{align*}
  \Delta_{\min} =
  \min_{k \in [K - 1]} \alpha(k) - \alpha(k + 1)
\end{align*}
be the \emph{minimum gap}. Then the $n$-step regret of $\bubblerank$ can be bounded as follows.

\begin{theorem}
\label{thm:upper bound} In any stochastic click bandit that satisfies \cref{sec:assumptions}, and for any $\delta \in (0, 1)$, the expected $n$-step regret of $\bubblerank$ is bounded as 
\begin{align*}
  &R(n) \leq \\
  &\quad 180 K \frac{\chi_{\max}}{\chi_{\min}} \frac{K - 1 + 2 \abs{\cV_0}}{\Delta_{\min}} \log(1 / \delta) +
  \delta^\frac{1}{2} K^3 n^2\,.
\end{align*}
\end{theorem}

\subsection{SAFETY}
\label{sec:safety}

Let
\begin{align}
  \mbox{}\hspace*{-4mm}
  \cV(\cR) {=}
  \set{(i, j) \in [K]^2: i < j, \cR^{-1}(i) > \cR^{-1}(j)}
  \hspace*{-3mm}\mbox{}
  \label{eq:safety}
\end{align}
be the set of \emph{incorrectly-ordered item pairs} in list $\cR$. Then our algorithm is safe in the following sense.

\begin{lemma}
\label{lem:safety} Let
\begin{align}
  \cV_0 =
  \cV(\cR_0)
  \label{eq:initial safety}
\end{align}
be the incorrectly-ordered item pairs in the initial base list $\cR_0$. Then the number of incorrectly-ordered item pairs in any displayed list $\rnd{\cR}_t$ is at most $\abs{\cV_0} + K / 2$, that is $\abs{\cV(\rnd{\cR}_t)} \leq \abs{\cV_0} + K / 2$ holds uniformly over time with probability of at least $1 - \delta^\frac{1}{2} K^2 n$.
\end{lemma}
\begin{proof}
Our claim follows from two observations. First, by the design of $\bubblerank$, any displayed list $\rnd{\cR}_t$ contains at most $K / 2$ item pairs that are ordered differently from its base list $\bar{\rnd{\cR}}_t$. Second, no base list $\bar{\rnd{\cR}}_t$ contains more incorrectly-ordered item pairs than $\cR_0$ with a high probability. In particular, under event $\cE$ in \cref{lem:concentration}, any change in the base list (line \ref{alg:28} of $\bubblerank$) reduces the number of incorrectly-order item pairs by one. In \cref{lem:concentration}, we prove that $\mathbb{P}(\cE) \geq 1 - \delta^\frac{1}{2} K^2 n$.
\end{proof}

\subsection{DISCUSSION}
\label{sec:discussion}

Our upper bound on the $n$-step regret of $\bubblerank$ (\cref{thm:upper bound}) is $O(\Delta_{\min}^{-1} \log n)$ for $\delta = n^{-4}$. This dependence is considered to be optimal in gap-dependent bounds. Our gap $\Delta_{\min}$ is the minimum difference in the attraction probabilities of items, and reflects the hardness of sorting the items by their attraction probabilities. This sorting problem is equivalent to the problem of learning $\cR^\ast$. So, a gap like $\Delta_{\min}$ is expected, and is the same as that in \citet{batchrank}.

Our regret bound is notable because it reflects two key characteristics of $\bubblerank$. First, the bound is linear in the number of incorrectly-ordered item pairs in the initial base list $\cR_0$. This suggests that $\bubblerank$ should have lower regret when initialized with a better list of items. We validate this dependence empirically in \cref{sec:experiments}. In many domains, such lists exist and are produced by existing ranking policies. They only need to be safely improved.

Second, the bound is $O(\chi_{\max} \chi_{\min}^{-1})$, where $\chi_{\max}$ and $\chi_{\min}$ are the maximum and minimum examination probabilities, respectively. In \cref{sec:sanity check on theorem}, we show that this dependence can be observed in problems where most attractive items are placed at infrequently examined positions. This limitation is intrinsic to $\bubblerank$, because attractive lower-ranked items cannot be placed at higher positions unless they are observed to be attractive at lower, potentially infrequently examined, positions.

The safety constraint of $\bubblerank$ is stated in \cref{lem:safety}. For $\delta = n^{-4}$, as discussed above, $\bubblerank$ becomes a rather safe algorithm, and is unlikely to display any list with more than $K / 2$ incorrectly-ordered item pairs than the initial base list $\cR_0$. More precisely, $\abs{\cV(\rnd{\cR}_t)} \leq \abs{\cV_0} + K / 2$ holds uniformly over time with probability of at least $1 - K^2 / n$. This safety feature of $\bubblerank$ is confirmed by our experiments in \cref{sec:safety results}.

The above discussion assumes that the time horizon $n$ is known. However, in practice, this is not always possible. We can extend $\bubblerank$ to the setting of an unknown time horizon by using the so-called \emph{doubling trick}~\citep[Section 2.3]{cesa-bianchi-2006-prediction}. Let $n$ be the estimated horizon. Then at time $n + 1$, $\bar{\mathcal{R}}_{n + 1}$ is set to $\mathcal{R}_0$ and $n$ is doubled. The statistics do not need to be reset.

$\bubblerank$  is computationally efficient. The time complexity of $\bubblerank$ is linear in the number of time steps and in each step $O(K)$ operations are required.

In this paper, we focus on re-ranking. But $\bubblerank$ can be extended to the full ranking problem as follows. Define $s_t(i, j)$ and $n_t(i, j)$ for all item pairs $(i, j)$. For even (odd) $K$ at odd (even) time steps, select a random item below position $K$ that has not been shown to be worse than the item at position $K$, and swap these items with probability $0.5$. The item that is not displayed gets feedback $0$. The rest of $\bubblerank$ remains the same. This algorithm can be analyzed in the same way as $\bubblerank$. 

\subsection{PROOF OF THEOREM~\ref{thm:upper bound}}
\label{sec:upper bound proof}

In \cref{lem:concentration}, we establish that there exists a favorable event $\cE$ that holds with probability $1 - \delta^\frac{1}{2} K^2 n$, when all beliefs $\rnd{s}_t(i, j)$ are at most $2 \sqrt{\rnd{n}_t(i, j) \log(1 / \delta)}$ from their respective means, uniformly for $i < j$ and $t \in [n]$. Since the maximum $n$-step regret is $K n$, we get that
\begin{align*}
  R(n) \leq
  \E{\hat{R}(n) \I{\cE}} + \delta^\frac{1}{2} K^3 n^2\,,
\end{align*}
where $\hat{R}(n) = \sum_{t = 1}^n r(\cR^\ast, \alpha, \chi) - r(\rnd{\cR}_t, \alpha, \chi)$. We bound $\hat{R}(n)$ next. For this, let
\begin{align*}
  \rnd{\cP}_t =
  \big\{(i, j) \in [K]^2: \,  i < j, \, \abs{\bar{\rnd{\cR}}_t^{-1}(i) - \bar{\rnd{\cR}}_t^{-1}(j)} = 1,& \\
   \rnd{s}_{t - 1}(i, j) \leq 2 \sqrt{\rnd{n}_{t - 1}(i, j) \log(1 / \delta)}\big\}&
\end{align*}
be the set of potentially randomized item pairs at time $t$. Then, by \cref{lem:off-base regret} on event $\cE$, which bounds the regret of list $\rnd{\cR}_t$ with the difference in the attraction probabilities of items $(i, j) \in \rnd{\cP}_t$, we have that
\begin{align*}
  &\hat{R}(n) \leq \\
  &~\quad 3 K \chi_{\max} \sum_{i = 1}^K \sum_{j = i + 1}^K
  \sum_{t = 1}^n (\alpha(i) - \alpha(j)) \I{(i, j) \in \rnd{\cP}_t}.
\end{align*}
Now note that for any randomized $(i, j) \in \rnd{\cP}_t$ at time $t$,
\begin{align*}
  \chi_{\min} (\alpha(i) - \alpha(j)) &\leq
  \Etp{\rnd{c}_t(\rnd{\cR}_t^{-1}(i)) - \rnd{c}_t(\rnd{\cR}_t^{-1}(j))} \\
  &=
  \Etp{\rnd{s}_t(i, j) - \rnd{s}_{t - 1}(i, j)},
\end{align*}
where $\Etp{\cdot}$ is the expectation conditioned on the history up to time $t$, $\rnd{\cR}_1, \rnd{c}_1, \ldots, \rnd{\cR}_{t - 1}, \rnd{c}_{t - 1}$. The inequality is from $\alpha(i) \geq \alpha(j)$, and Assumptions~\ref{ass:order-free examination} and \ref{ass:examination scaling}. The above two inequalities yield
\begin{align*}
  \hat{R}(n)
  & \leq 6 K \frac{\chi_{\max}}{\chi_{\min}} \sum_{i = 1}^K \sum_{j = i + 1}^K
  \sum_{t = 1}^n \\
  &\quad\qquad \Etp{\rnd{s}_t(i, j) - \rnd{s}_{t - 1}(i, j)} \I{(i, j) \in \rnd{\cP}_t} \\
  & \leq 6 K \frac{\chi_{\max}}{\chi_{\min}} \sum_{i = 1}^K \sum_{j = i + 1}^K
  \I{\exists t \in [n]: (i, j) \in \rnd{\cP}_t} \times \\
  & \quad\qquad \sum_{t = 1}^n \Etp{\rnd{s}_t(i, j) - \rnd{s}_{t - 1}(i, j)},
\end{align*} 
where the extra factor of two is because $\bubblerank$ randomizes any pair of items $(i, j) \in \rnd{\cP}_t$ at least once in any two consecutive steps. Moreover, for any $i < j$ on event $\cE$,
\begin{align*}
  &\sum_{t = 1}^n (\rnd{s}_t(i, j) - \rnd{s}_{t - 1}(i, j)) =
  \rnd{s}_n(i, j)\\
  &~\quad \leq 15 \frac{\alpha(i) + \alpha(j)}{\alpha(i) - \alpha(j)} \log(1 / \delta) 
 \leq \frac{30}{\Delta_{\min}} \log(1 / \delta).
\end{align*}
The first inequality is by \cref{lem:cumulative click loss}, which establishes that the maximum difference in clicks of any randomized pair of items is bounded. After that, the better item is found and the pair of items is never randomized again. The last inequality is by $\alpha(i) + \alpha(j) \leq 2$ and $\alpha(i) - \alpha(j) \geq \Delta_{\min}$. Now we chain the above two inequalities and get that
\begin{align*}
  \hat{R}(n) \leq{}
  &180 K \frac{\chi_{\max}}{\chi_{\min}} \frac{1}{\Delta_{\min}} \log(1 / \delta) \times \\
  &\sum_{i = 1}^K \sum_{j = i + 1}^K \I{\exists t \in [n]: (i, j) \in \rnd{\cP}_t}.
\end{align*}
Finally, let $\rnd{\cP} = \bigcup_{t \in [n]} \rnd{\cP}_t$. Then, on event $\cE$, $\abs{\rnd{\cP}} \leq K - 1 + 2 \abs{\cV_0}$. This follows from the design of $\bubblerank$ (\cref{lem:swap bound}) and completes the proof. \qed

%% file: Experiments.tex

\section{EXPERIMENTAL RESULTS}
\label{sec:experiments}

We conduct four experiments to evaluate $\bubblerank$. 
In \cref{sec:experimental setup}, we describe our experimental setup. In \cref{sec:results with regret}, we report the regret of compared algorithms, which measures the rate of convergence to the optimal list in hindsight. In \cref{sec:safety results}, we validate the safety of $\bubblerank$. In \cref{sec:sanity check on theorem}, we validate the tightness of the regret bound in \cref{thm:upper bound}. 
Due to space limitations, we report the \emph{Normalized Discounted Cumulative Gain}~(NDCG)
of compared algorithms, which measures the quality of displayed lists, in~\cref{sec:results with ndcg}.

\begin{figure*}[th]
	\includegraphics{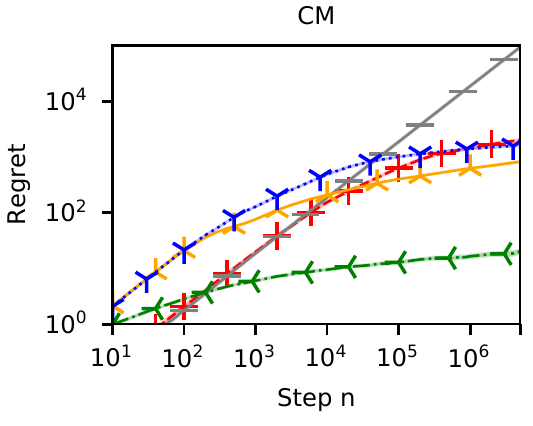}
	\includegraphics{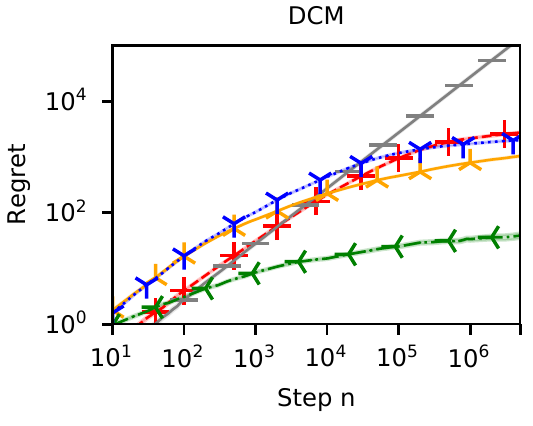}
	\includegraphics{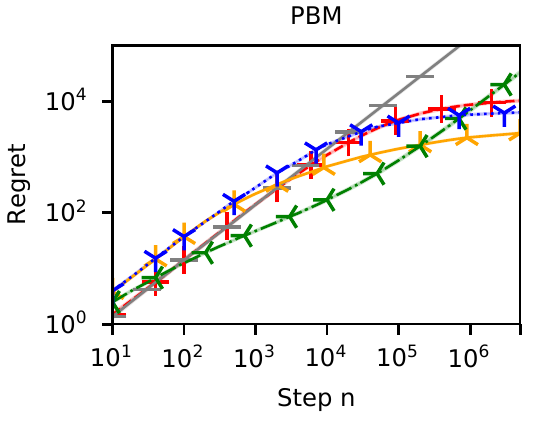}
	\caption{The $n$-step regret of $\bubblerank$ (red), $\cascadeklucb$ (green), $\batchrank$ (blue), $\toprank$ (orange), and $\baseline$ (grey) in the CM, DCM, and PBM in up to $5$ million steps. Lower is better. The results are averaged over all $100$ queries and $10$ runs per query.  The shaded regions represent standard errors of our estimates.}
	\label{fig:all queries}
\end{figure*}

\begin{figure*}[th]
	\includegraphics{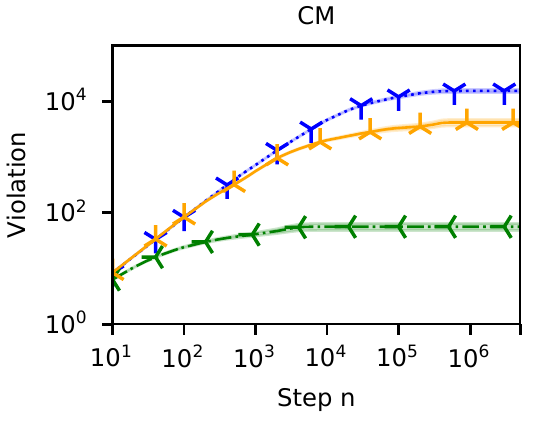}	
	\includegraphics{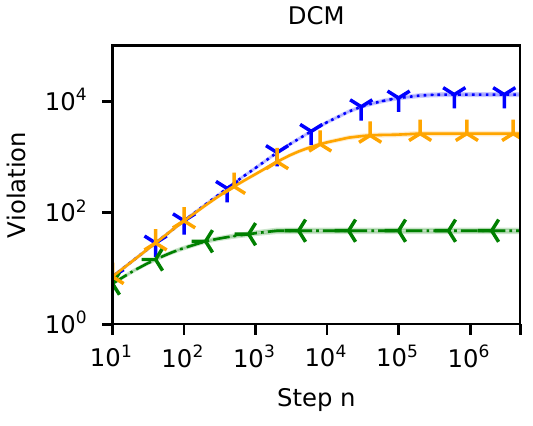}
	\includegraphics{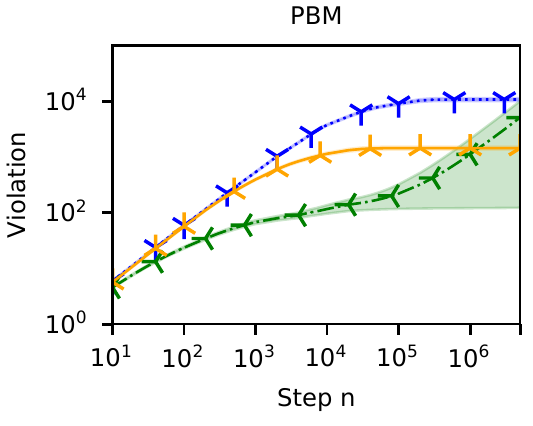}
	\caption{The $n$-step violation of the safety constraint of $\bubblerank$ by $\cascadeklucb$ (green), $\batchrank$ (blue), and $\toprank$ (orange) in the CM, DCM, and PBM in up to $5$ million steps. Lower is better. The shaded regions represent standard errors of our estimates.}
	\label{fig:violations}
\end{figure*}

\subsection{EXPERIMENTAL SETUP}
\label{sec:experimental setup}

We evaluate $\bubblerank$ on the \emph{Yandex} click  dataset.\footnote{\url{https://academy.yandex.ru/events/data_analysis/relpred2011}} The dataset contains user search sessions from the log of the \emph{Yandex} search engine. It is the largest publicly available dataset containing user clicks, with more than $30$ million search sessions. Each session contains at least one search query together with $10$ ranked items.

We preprocess the dataset as in \citep{batchrank}. In particular, we randomly select $100$ frequent search queries, and then learn the parameters of three click models using the PyClick\footnote{\url{https://github.com/markovi/PyClick}} package: CM and PBM, described in \cref{sec:click models}, as well as the \emph{dependent click model (DCM)}~\cite{guo09efficient}. 

The DCM is an extension of the CM~\cite{craswell08experimental} where each position $k$ is associated with an abandonment probability $v(k)$. When the user clicks on an item at position $k$, the user stops scanning the list with probability $v(k)$. Therefore, the DCM can model multiple clicks. Following the work in \cite{katariya16dcm}, we incorporate abandonment into our definition of reward for DCM and define it as the number of abandonment clicks. 
The \emph{abandonment click} is a click after which a user stops browsing the list, and each time step 
contains at most one abandonment click. 
So, the expected reward for DCM equals the probability of abandonment clicks, which is computed as follows: 
\begin{align*}
r(\cR, \alpha, \chi) & = \sum_{k = 1}^K \chi(\cR, k) v(k) \alpha(\cR(k)), \\
\chi(\cR, k) &= \prod_{i = 1}^{k - 1} (1 - v(i)) \alpha(\cR(i))
\end{align*}
is the examination probability of position $k$ in list $\cR$.  A high reward means a user stops the search session because of clicking on an item with high attraction probability.

The learned CM, DCM, and PBM  are used to simulate user click feedback. We experiment with multiple click models to show the robustness of $\bubblerank$ to multiple models of user feedback.

For each query, we choose $10$ items. The number of positions is equal to the number of items, $K = L = 10$. The objective of our re-ranking problem is to place $5$ most attractive items in the descending order of their attractiveness at the $5$ highest positions, as in \citep{batchrank}. The performance of $\bubblerank$ and our baselines is also measured only at the top $5$ positions.

$\bubblerank$ is compared to three baselines $\tt Cascade\mhyphen$ $\tt KL\mhyphen UCB$~\cite{cascadingbandit}, $\batchrank$~\cite{batchrank}, and $\toprank$~\cite{lattimore2018toprank}. The former is near optimal in the CM~\cite{cascadingbandit}, but can have linear regret in other click models. Note that linear regret arises when $\cascadeklucb$ erroneously converges to a suboptimal ranked list. $\batchrank$ and $\toprank$ can learn the optimal list $\cR^\ast$ in a wide range of click models, including the CM, DCM, and PBM. However, they can perform poorly in early stages of learning because they randomly shuffles displayed lists to average out the position bias. All experiments are run for $5$ million steps, after which at least two algorithms converge to the optimal ranked list. 

In the \emph{Yandex} dataset, each query is associated with many different ranked lists, due to the presence of various personalization features of the production ranker. We take the most frequent ranked list for each query as the initial base list $\cR_0$ in $\bubblerank$, since we assume that the most frequent ranked list is what the production ranker would produce in the absence of any personalization. We also compare $\bubblerank$ to a production baseline, called $\baseline$, where the initial list $\cR_0$ is applied for $n$ steps. 

\subsection{RESULTS WITH REGRET}
\label{sec:results with regret}

In the first experiment, we compare $\bubblerank$ to $\cascadeklucb$, $\batchrank$, and $\toprank$ in the CM, DCM, and PBM of all $100$ queries. 
Among them, $\toprank$ is the state-of-the-art online LTR algorithm in multiple click models. 
We evaluate these algorithms by their cumulative regret, which is defined in \eqref{eq:regret}, at the top $5$ positions. The regret,  a measure of convergence,  is a widely-used metric in the bandit literature~\cite{auer02finitetime,katariya16dcm,cascadingbandit,batchrank}. In the CM and PBM, the regret is the cumulative loss in clicks when a sequence of learned lists is compared to the optimal list in hindsight. In the DCM, the regret is the cumulative loss in abandonment clicks. We also report the regret of $\baseline$. 

Our results are reported in \cref{fig:all queries}. 
We observe that the regret of  $\baseline$ grows linearly with time $n$, which means that it is not optimal on average. 
$\cascadeklucb$ learns $\cR^\ast$ quickly in both the CM and DCM, but has linear regret in the PBM. 
This is expected since  $\cascadeklucb$ is designed for the CM, and the DCM is an extension of the CM. 
As for the PBM, which is beyond the modeling assumptions of $\cascadeklucb$, there is no guarantee on the performance of $\cascadeklucb$.  
$\bubblerank$, $\batchrank$, and $\toprank$ can learn in all three click models. 
Compared to $\batchrank$ and $\toprank$, $\bubblerank$ has a higher regret in $5$ million steps. 
However, in earlier steps, $\bubblerank$ has a lower regret than $\batchrank$ and $\toprank$, as it takes advantage of the initial base list $\cR_0$. 
In general, these results show that $\bubblerank$ converges to the optimal list slower than $\batchrank$ and $\toprank$. This is expected because $\bubblerank$ is designed to be a safe algorithm, and only learns better lists by exchanging neighboring items in the base list. 

\subsection{SAFETY RESULTS}
\label{sec:safety results}

\begin{figure*}[t]
	\includegraphics{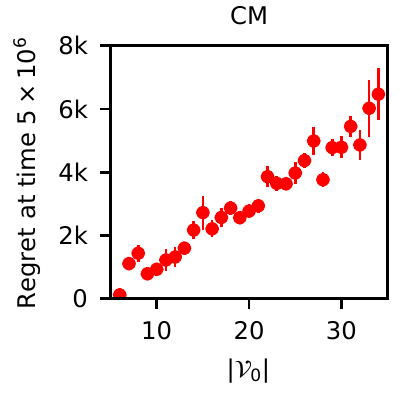}
	\includegraphics{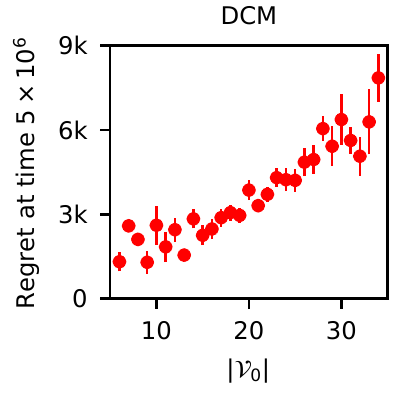}
	\includegraphics{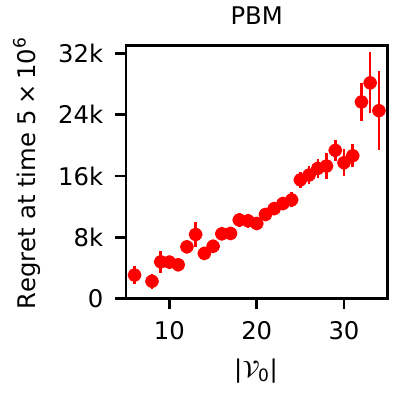}
	\includegraphics{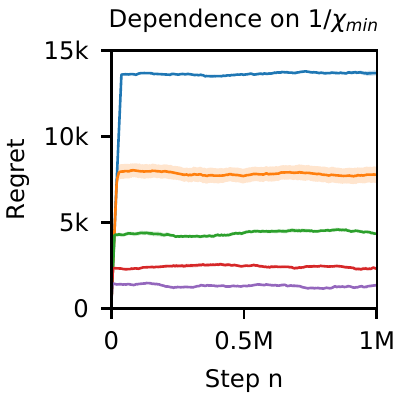}
	\caption{Regret of $\bubblerank$ as a function of the number of  incorrectly-ordered item pairs 
	$\abs{\cV_0}$, and the minimal examination probability  $\chi_{\min}$. In the right plot, the purple, red, green, orange, and blue colors represent $\chi_{\min}$ equals $0.5$, $0.5^2$, $0.5^3$, $0.5^4$, and $0.5^5$, respectively. The shaded regions represent standard errors of our estimates. }
	\label{fig:initialization}
\end{figure*}

In the previous experiment, we show that $\bubblerank$ does not learn as fast as $\cascadeklucb$, $\batchrank$, and $\toprank$ because it operates under the additional safety constraint in \cref{lem:safety}. The constraint is that $\bubblerank$ is unlikely to display any list with more than $K / 2$ incorrectly-ordered item pairs than the initial base list $\cR_0$. More precisely, $\abs{\cV(\rnd{\cR}_t)} \leq \abs{\cV(\cR_0)} + K / 2$ holds uniformly over time with probability of at least $1 - K^2 / n$ for $\delta = n^{-4}$ (\cref{sec:discussion}), where $\cV$ is defined in \eqref{eq:safety}. In the second experiment, we answer the question how often do $\bubblerank$, $\cascadeklucb$,  $\batchrank$, and $\toprank$ 
violate this constraint empirically.
We define the \emph{safety constraint violation in $n$ steps} as
\begin{align}
  V(n) =
  \sum_{t = 1}^{n} \I{\abs{\cV(\rnd{\cR}_t)} > \abs{\cV(\cR_0)} + K / 2},
  \label{eq:constraint violation}
\end{align}
where $\rnd{\cR}_t$ is the displayed list at time $t$. 

We report the $n$-step safety constraint violation of $\cascadeklucb$, $\batchrank$, and $\toprank$ in \cref{fig:violations}. We do not include results of $\bubblerank$ since $\bubblerank$ never violates the constraint in our experiments. We observe that the safety constraint violations of $\cascadeklucb$ in the first $100$ steps are $24.12 \pm 0.76$, $ 23.33 \pm 0.90$, and $23.63 \pm 0.96$ in the CM, DCM, and PBM, respectively. Translating this to a search scenario, $\cascadeklucb$ may show unsafe results, which are significantly worse than the initial base list $\cR_0$, and may hurt user experience, more than $20\%$ of search sessions in the first $100$ steps. Even worse, the violations of $\cascadeklucb$ grow linearly with time in the PBM. The safety issues of $\batchrank$, and $\toprank$ are more severe than that of $\cascadeklucb$.  More precisely, the violations of $\batchrank$ in the first $100$ steps are $83.01 \pm 0.56$, $71.89 \pm0.92$, and $59.63 \pm 0.98$ in the CM, DCM, and PBM, respectively.
And the violations of $\toprank$ are $83.56 \pm 0.70$, $71.47 \pm1.07$, and $57.00 \pm 1.24$ in the CM, DCM, and PBM, respectively. 
Note that the performance of $\toprank$ is close to that of $\batchrank$ in the first $100$ steps since they both require the ranked lists to be randomly shuffled during the initial stages. 
Thus, $\batchrank$, and $\toprank$  would frequently hurt the user experience  during the early stages of learning.

To conclude, $\bubblerank$ learns without violating its safety constraint, while $\cascadeklucb$, $\batchrank$, and $\toprank$  violate the constraint frequently. 
Together with results in~\cref{sec:results with regret}, $\bubblerank$ is a safe algorithm but, to satisfy the safety constraint, it compromises the performance and learns slower than $\batchrank$ and $\toprank$.  
In~\cref{sec:results with ndcg}, we compare $\bubblerank$ to baselines in NDCG and  show that $\bubblerank$ converges to the optimal lists in hindsight.

\subsection{SANITY CHECK ON THE REGRET BOUND}
\label{sec:sanity check on theorem}

We prove an upper bound on the $n$-step regret of $\bubblerank$ in \cref{thm:upper bound}. In comparison to the upper bounds of $\cascadeklucb$ \cite{cascadingbandit} and $\batchrank$ \cite{batchrank}, we have two new problem-specific constants: $\abs{\cV_0}$ and $1 / \chi_{\min}$. In this section, we show that these constants are intrinsic to the behavior of $\bubblerank$.

We first study how the number of incorrectly-ordered item pairs in the initial base list $\cR_0$, $\abs{\cV_0}$, impacts the regret of $\bubblerank$. We choose $10$ random initial base lists $\cV_0$ in each of our $100$ queries and plot the regret of $\bubblerank$ as a function of $\abs{\cV_0}$. Our results are shown in \cref{fig:initialization}. We observe that the regret of $\bubblerank$ is linear in $\abs{\cV_0}$ in the CM, DCM, and PBM. This is the same dependence as in our regret bound (\cref{thm:upper bound}). 

We then study the impact of the minimum examination probability $\chi_{\min}$ on the regret of $\bubblerank$. We experiment with a synthetic PBM with $10$ items, which is parameterized by $\alpha = (0.9, 0.5, \dots, 0.5)$ and $\chi = (0.9, \dots, 0.9, 0.5^i, 0.5^i)$ for $i \geq 1$. The most attractive item is placed at the last position in $\cR_0$, $\cR_0 = (2, \dots, K - 1, 1)$. Since this position is examined with probability $0.5^i$, we expect the regret to double when $i$ increases by one. We experiment with $i \in [5]$ in \cref{fig:initialization} and observe this trend in $1$ million steps. This confirms that the dependence on $1 / \chi_{\min}$ in \cref{thm:upper bound} is generally unavoidable.

%% file: RelatedWork.tex

\section{RELATED WORK}
\label{sec:related work}

%
%
%

Online LTR via click feedback has been mainly studied in two approaches: under specific click models~\cite{combes15learning,katariya16dcm,cascadingbandit,kveton15combinatorial,lagree16multipleplay,zong16cascading}; or without a particular assumption on click models~\cite{lattimore2018toprank,radlinski08learning,slivkins13ranked,batchrank}. 
Algorithms from the first group efficiently learn optimal rankings in the their considered click models but do not have guarantees beyond their specific click models. 
Algorithms from the second group, on the other hand, learn the optimal rankings in a broader class of click models. 
$\toprank$~\cite{lattimore2018toprank} is the state-of-the-art of the second group, which has the regret of $O(K^2\log(n))$ in our re-ranking setup, that is $L=K$. 
$\bubblerank$ also belongs to the second group and the regret of $\bubblerank$ is comparable to that of $\toprank$ given a good initial list, when $| \cV_0 | = O(K)$. 
However, unlike $\bubblerank$, $\toprank$ and all the previous algorithms do not consider safety. 
They explore aggressively in the initial steps and may display irrelevant items at high positions, which may then hurt user experiences~\cite{Wang:2018}.  

Our safety problem is related to the warm start problem~\cite{strehl2010learning}. 
Contextual bandits~\cite{agarwal2014taming,li2010contextual,Moon:2010} 
deal with a broader class of models than we do and are used to address the warm start problem. But they are limited to small action sets, and thus unsuitable for the ranking setup that we consider in this paper.

The warm start LTR has been studied in multiple papers~\cite{hofmann2013reusing,Vorobev2015Gathering,yanpractical}, where the goal is to use an online algorithm to fine tune the results generated by an offline-trained ranker. 
In these papers, different methods for learning prior distributions of Thompson sampling based online LTR algorithms from offline datasets have been proposed. 
However, these methods have the following drawbacks. 
First, the offline data may not well align with user preferences~\cite{clicklambda}, which may result in a biased prior assumption. 
Second, grid search with online A/B tests may alleviate this and find a proper prior assumption~\cite{Vorobev2015Gathering}, but the online A/B test requires additional costs in terms of user experience. 
Third,  there is no safety constraint in these methods. Even with carefully picked priors, they may recommend irrelevant items to users, e.g., new items with little prior knowledge. 
In contrast, $\bubblerank$ starts from the production ranked list and learns under the safety constraint. Thus, $\bubblerank$ gets rid of these drawbacks.  

Another related line of work are conservative bandits~\cite{kazerouni2017conservative,wu2016conservative}. In conservative bandits, the learned policy is safe in the sense that its expected \emph{cumulative} reward is at least $1 - \alpha$ fraction of that of the baseline policy with high probability. This notion of safety is less stringent than that in our work (\cref{sec:safety}). In particular, our notion of safety is \emph{per-step}, in the sense that any displayed list is only slightly worse than the initial base list with a high probability. We do not compare to conservative bandits in our experiments because existing algorithms for conservative bandits require the action space to be small. The actions in our problem are ranked lists, and their number is exponential in $K$.


%% file: Conclusions.tex

\section{CONCLUSIONS}
\label{sec:conclusions}

In this paper, we fill a gap in the LTR literature by proposing $\bubblerank$, a re-ranking algorithm that gradually improves an initial base list, which we assume to be provided by an offline LTR approach. The improvements are learned from small perturbations of base lists, which are unlikely to degrade the user experience greatly. We prove a gap-dependent upper bound on the regret of $\bubblerank$ and evaluate it on a large-scale click dataset from a commercial search engine.

We leave open several questions of interest. For instance, our paper studies $\bubblerank$ in the setting of re-ranking. 
Although we explain an approach of extending $\bubblerank$ to the general ranking setup in \cref{sec:discussion}, we expect further experiments to validate this approach. 
Our general topic of interest are exploration schemes that are more conservative than those of existing online LTR methods. Existing methods are not very practical because they can explore highly irrelevant items at frequently examined positions.

%% file: Appendix.tex

\clearpage
\onecolumn
\appendix

\section{LEMMAS}
\label{sec:technical lemmas}

\begin{lemma}
\label{lem:list regret} Let $\cR$ be any list over $[K]$. Let
\begin{align}
  \Delta(\cR) =
  \sum_{k = 1}^{K - 1} \I{\alpha(\cR(k + 1)) - \alpha(\cR(k)) > 0} \times
  (\alpha(\cR(k + 1)) - \alpha(\cR(k)))
  \label{eq:attraction gap} 
\end{align}
be the \emph{attraction gap} of list $\cR$. Then the expected regret of $\cR$ is bounded as
\begin{align*}
  \sum_{k = 1}^K (\chi(\cR^\ast, k) \alpha(k) - \chi(\cR, k) \alpha(\cR(k))) \leq
  K \chi_{\max} \Delta(\cR)\,.
\end{align*}
\end{lemma}
\begin{proof}
Fix position $k \in [K]$. Then
\begin{align*}
  \chi(\cR^\ast, k) \alpha(k) - \chi(\cR, k) \alpha(\cR(k))
  & \leq \chi(\cR^\ast, k) (\alpha(k) - \alpha(\cR(k))) \\
  & \leq \chi_{\max} (\alpha(k) - \alpha(\cR(k)))\,,
\end{align*}
where the first inequality follows from the fact that the examination probability of any position is the lowest in the optimal list (Assumption~\ref{ass:lowest examination}) and the second inequality follows from the definition of $\chi_{\max}$. In the rest of the proof, we bound $\alpha(k) - \alpha(\cR(k))$. We consider three cases. First, let $\alpha(\cR(k)) \geq \alpha(k)$. Then $\alpha(k) - \alpha(\cR(k)) \leq 0$ and can be trivially bounded by $\Delta(\cR)$. Second, let $\alpha(\cR(k)) < \alpha(k)$ and $\pi(k) > k$, where $\pi(k)$ is the position of item $k$ in list $\cR$. Then
\begin{align*}
  \alpha(k) - \alpha(\cR(k)) & =
  \alpha(\cR(\pi(k))) - \alpha(\cR(k)) \\
  & \leq  \sum_{i = k}^{\pi(k) - 1} \I{\alpha(\cR(i + 1)) - \alpha(\cR(i)) > 0} \alpha(\cR(i + 1)) - \alpha(\cR(i))).
\end{align*}
From the definition of $\Delta(\cR)$, this quantity is bounded from above by $\Delta(\cR)$. Finally, let $\alpha(\cR(k)) < \alpha(k)$ and $\pi(k) < k$. This implies that there exists an item at a lower position than $k$, $j > k$, such that $\alpha(\cR(j)) \geq \alpha(k)$. Then
\begin{align*}
  \alpha(k) - \alpha(\cR(k)) & \leq
  \alpha(\cR(j)) - \alpha(\cR(k)) \\
  & \leq \sum_{i = k}^{j - 1} \I{\alpha(\cR(i + 1)) - \alpha(\cR(i)) > 0} (\alpha(\cR(i + 1)) - \alpha(\cR(i)))\,.
\end{align*}
From the definition of $\Delta(\cR)$, this quantity is bounded from above by $\Delta(\cR)$. This concludes the proof.
\end{proof}

\begin{lemma}
\label{lem:off-base gap} Let
\begin{align*}
  \rnd{\cP}_t =
  \big\{(i, j) \in [K]^2: \, & i < j, \, \abs{\bar{\rnd{\cR}}_t^{-1}(i) - \bar{\rnd{\cR}}_t^{-1}(j)} = 1, 
 \rnd{s}_{t - 1}(i, j) \leq 2 \sqrt{\rnd{n}_{t - 1}(i, j) \log(1 / \delta)}\big\}
\end{align*}
be the set of potentially randomized item pairs at time $t$ and $\rnd{\Delta}_t = \max_{\rnd{\cR}_t} \Delta(\rnd{\cR}_t)$ be the \emph{maximum attraction gap} of any list $\rnd{\cR}_t$, where $\Delta(\rnd{\cR}_t)$ is defined in \eqref{eq:attraction gap}. Then on event $\cE$ in \cref{lem:concentration},
\begin{align*}
  \rnd{\Delta}_t \leq
  3 \sum_{i = 1}^K \sum_{j = i + 1}^K \I{(i, j) \in \rnd{\cP}_t} (\alpha(i) - \alpha(j))
\end{align*}
holds at any time $t \in [n]$.
\end{lemma}
\begin{proof}
Fix list $\rnd{\cR}_t$ and position $k \in [K - 1]$. Let $i', i, j, j'$ be items at positions $k - 1, k, k + 1, k + 2$ in $\bar{\rnd{\cR}}_t$. If $k = 1$, let $i' = i$; and if $k = K - 1$, let $j' = j$. We consider two cases.

First, suppose that the permutation at time $t$ is such that $i$ and $j$ could be exchanged. Then
\begin{align*}
 \alpha(\rnd{\cR}_t^{-1}(k + 1)) - \alpha(\rnd{\cR}_t^{-1}(k)) \leq  \I{(\min \set{i, j}, \max \set{i, j}) \in \rnd{\cP}_t} (\alpha(\min \set{i, j}) - \alpha(\max \set{i, j}))
\end{align*}
holds on event $\cE$ by the design of $\bubblerank$. More specifically, $(\min \set{i, j}, \max \set{i, j}) \notin \rnd{\cP}_t$ implies that $\alpha(\rnd{\cR}_t^{-1}(k + 1)) - \alpha(\rnd{\cR}_t^{-1}(k)) \leq 0$.

Second, suppose that the permutation at time $t$ is such that $i$ and $i'$ could be exchanged, $j$ and $j'$ could be exchanged, or both. Then
\begin{align*}
 \alpha(\rnd{\cR}_t^{-1}(k + 1)) - \alpha(\rnd{\cR}_t^{-1}(k)) \leq {}
  & \I{(\min \set{i, i'}, \max \set{i, i'}) \in \rnd{\cP}_t} (\alpha(\min \set{i, i'}) - \alpha(\max \set{i, i'})) + {} \\
  & \alpha(j) - \alpha(i) + {} \\
  & \I{(\min \set{j, j'}, \max \set{j, j'}) \in \rnd{\cP}_t} (\alpha(\min \set{j, j'})  - \alpha(\max \set{j, j'}))
\end{align*}
holds by the same argument as in the first case. Also note that
\begin{align*}
 \alpha(j) - \alpha(i) \leq 
 \I{(\min \set{i, j}, \max \set{i, j}) \in \rnd{\cP}_t} (\alpha(\min \set{i, j}) - \alpha(\max \set{i, j}))
\end{align*}
holds on event $\cE$ by the design of $\bubblerank$. Therefore, for any position $k \in [K - 1]$ in both above cases,
\begin{align*}
  \alpha(\rnd{\cR}_t^{-1}(k + 1)) - \alpha(\rnd{\cR}_t^{-1}(k)) \leq\!
  \sum_{\ell = k - 1}^{k + 1} 
  & \I{\left(\min \set{\bar{\rnd{\cR}}_t^{-1}(\ell), \bar{\rnd{\cR}}_t^{-1}(\ell + 1)},
  \max \set{\bar{\rnd{\cR}}_t^{-1}(\ell), \bar{\rnd{\cR}}_t^{-1}(\ell + 1)}\right) \in \rnd{\cP}_t} \times \\
  & \left(\alpha\left(\min \set{\bar{\rnd{\cR}}_t^{-1}(\ell), \bar{\rnd{\cR}}_t^{-1}(\ell + 1)}\right)- 
  \alpha\left(\max \set{\bar{\rnd{\cR}}_t^{-1}(\ell), \bar{\rnd{\cR}}_t^{-1}(\ell + 1)}\right)\right).
\end{align*}
Now we sum over all positions and note that each pair of $\bar{\rnd{\cR}}_t^{-1}(\ell)$ and $\bar{\rnd{\cR}}_t^{-1}(\ell + 1)$ appears on the right-hand side at most three times, in any list $\rnd{\cR}_t$. This concludes our proof.
\end{proof}

\begin{lemma}
\label{lem:off-base regret} Let $\rnd{\cP}_t$ be defined as in \cref{lem:off-base gap}. Then on event $\cE$ in \cref{lem:concentration},
\begin{align*}
  \sum_{k = 1}^K (\chi(\cR^\ast, k) \alpha(k) - \chi(\rnd{\cR}_t, k) \alpha(\rnd{\cR}_t(k))) \leq 
   3 K \chi_{\max} \sum_{i = 1}^K \sum_{j = i + 1}^K \I{(i, j) \in \rnd{\cP}_t} (\alpha(i) - \alpha(j))
\end{align*}
holds at any time $t \in [n]$.
\end{lemma}
\begin{proof}
A direct consequence of \cref{lem:list regret,lem:off-base gap}.
\end{proof}

\begin{lemma}
\label{lem:swap bound} Let $\rnd{\cP}_t$ be defined as in \cref{lem:off-base gap}, $\rnd{\cP} = \bigcup_{t = 1}^n \rnd{\cP}_t$, and $\cV_0$ be defined as in \eqref{eq:initial safety}. Then on event $\cE$ in \cref{lem:concentration},
\begin{align*}
  \abs{\rnd{\cP}} \leq K - 1 + 2 \abs{\cV_0}\,.
\end{align*}
\end{lemma}
\begin{proof}
From the design of $\bubblerank$, $\abs{\rnd{\cP}_1} = K - 1$. The set of randomized item pairs grows only if the base list in $\bubblerank$ changes. When this happens, the number of incorrectly-ordered item pairs decreases by one, on event $\cE$, and the set of randomized item pairs increases by at most two pairs. This event occurs at most $\abs{\cV_0}$ times. This concludes our proof.
\end{proof}

\begin{lemma}
\label{lem:cumulative click loss} For any items $i$ and $j$ such that $i < j$,
\begin{align*}
  \rnd{s}_n(i, j) \leq 15 \frac{\alpha(i) + \alpha(j)}{\alpha(i) - \alpha(j)} \log(1 / \delta)
\end{align*}  
on event $\cE$ in \cref{lem:concentration}.
\end{lemma}
\begin{proof}
To simplify notation, let $\rnd{s}_t = \rnd{s}_t(i, j)$ and $\rnd{n}_t = \rnd{n}_t(i, j)$. The proof has two parts. First, suppose that $\rnd{s}_t \leq 2 \sqrt{\rnd{n}_t \log(1 / \delta)}$ holds at all times $t \in [n]$. Then from this assumption and on event $\cE$ in \cref{lem:concentration},
\begin{align*}
  \frac{\alpha(i) - \alpha(j)}{\alpha(i) + \alpha(j)} \rnd{n}_t - 2 \sqrt{\rnd{n}_t \log(1 / \delta)} \leq
  \rnd{s}_t \leq
  2 \sqrt{\rnd{n}_t \log(1 / \delta)}\,.
\end{align*}
This implies that
\begin{align*}
  \rnd{n}_t \leq
  \left[4 \frac{\alpha(i) + \alpha(j)}{\alpha(i) - \alpha(j)}\right]^2 \log(1 / \delta)
\end{align*}
at any time $t$, and in turn that
\begin{align*}
  \rnd{s}_t \leq
  2 \sqrt{\rnd{n}_t \log(1 / \delta)} \leq
  8 \frac{\alpha(i) + \alpha(j)}{\alpha(i) - \alpha(j)} \log(1 / \delta)
\end{align*}
at any time $t$. Our claim follows from setting $t = n$.

Now suppose that $\rnd{s}_t \leq 2 \sqrt{\rnd{n}_t \log(1 / \delta)}$ does not hold at all times $t \in [n]$. Let $\tau$ be the first time when $\rnd{s}_\tau > 2 \sqrt{\rnd{n}_\tau \log(1 / \delta)}$. Then from the definition of $\tau$ and on event $\cE$ in \cref{lem:concentration},
\begin{align*}
  \frac{\alpha(i) - \alpha(j)}{\alpha(i) + \alpha(j)} \rnd{n}_\tau - 2 \sqrt{\rnd{n}_\tau \log(1 / \delta)}
  & \leq \rnd{s}_\tau \leq
  \rnd{s}_{\tau - 1} + 1 \\
  & \leq 2 \sqrt{\rnd{n}_\tau \log(1 / \delta)} + 1 \\
  & \leq 3 \sqrt{\rnd{n}_\tau \log(1 / \delta)}\,,
\end{align*}
where the last inequality holds for any $\delta \leq 1 / e$. This implies that
\begin{align*}
  \rnd{n}_\tau \leq
  \left[5 \frac{\alpha(i) + \alpha(j)}{\alpha(i) - \alpha(j)}\right]^2 \log(1 / \delta)\,,
\end{align*}
and in turn that
\begin{align*}
  \rnd{s}_\tau \leq
  3 \sqrt{\rnd{n}_\tau \log(1 / \delta)} \leq
  15 \frac{\alpha(i) + \alpha(j)}{\alpha(i) - \alpha(j)} \log(1 / \delta)\,.
\end{align*}
Now note that $\rnd{s}_t = \rnd{s}_\tau$ for any $t > \tau$, from the design of $\bubblerank$. This concludes our proof.
\end{proof}

For some $\cF_t = \sigma(\rnd{\cR}_1, \rnd{c}_1, \dots, \rnd{\cR}_t, \rnd{c}_t)$-measurable event $A$, let $\mathbb{P}_t(A) = \mathbb{P}(A \mid \cF_t)$ be the conditional probability of $A$ given history $\rnd{\cR}_1, \rnd{c}_1, \dots, \rnd{\cR}_t, \rnd{c}_t$. Let the corresponding conditional expectation operator be $\Et{\cdot}$. Note that $\bar{\rnd{\cR}}_t$ is $\cF_{t - 1}$-measurable.

\begin{lemma}
\label{lem:click loss} Let $i, j \in [K]$ be any items at consecutive positions in $\bar{\rnd{\cR}}_t$ and
\begin{align*}
  \rnd{z} = \rnd{c}_t(\rnd{\cR}_t^{-1}(i)) - \rnd{c}_t(\rnd{\cR}_t^{-1}(j))\,.
\end{align*}
Then, on the event that $i$ and $j$ are subject to randomization at time $t$,
\begin{align*}
  \Etp{\rnd{z} \mid \rnd{z} \neq 0} \geq
  \frac{\alpha(i) - \alpha(j)}{\alpha(i) + \alpha(j)}
\end{align*}
when $\alpha(i) > \alpha(j)$, and $\Etp{- \rnd{z} \mid \rnd{z} \neq 0} \leq 0$ when $\alpha(i) < \alpha(j)$.
\end{lemma}
\begin{proof}
The first claim is proved as follows. From the definition of expectation and $\rnd{z} \in \set{-1, 0, 1}$,
\begin{align*}
  \Etp{\rnd{z} \mid \rnd{z} \neq 0}
  & = \frac{\mathbb{P}_{t - 1}(\rnd{z} = 1, \rnd{z} \neq 0) - \mathbb{P}_{t - 1}(\rnd{z} = -1, \rnd{z} \neq 0)}
  {\mathbb{P}_{t - 1}(\rnd{z} \neq 0)} \\
  & = \frac{\mathbb{P}_{t - 1}(\rnd{z} = 1) - \mathbb{P}_{t - 1}(\rnd{z} = -1)}{\mathbb{P}_{t - 1}(\rnd{z} \neq 0)} \\
  & =  \frac{\Etp{\rnd{z}}}{\mathbb{P}_{t - 1}(\rnd{z} \neq 0)}\,,
\end{align*}
where the last equality is a consequence of $\rnd{z} = 1 \implies \rnd{z} \neq 0$ and that $\rnd{z} = -1 \implies \rnd{z} \neq 0$.

Let $\chi_i = \Etp{\chi(\rnd{\cR}_t, \rnd{\cR}_t^{-1}(i))}$ and $\chi_j = \Etp{\chi(\rnd{\cR}_t, \rnd{\cR}_t^{-1}(j))}$ denote the average examination probabilities of the positions with items $i$ and $j$, respectively, in $\rnd{\cR}_t$; and consider the event that $i$ and $j$ are subject to randomization at time $t$. By Assumption~\ref{ass:order-free examination}, the values of $\chi_i$ and $\chi_j$ do not depend on the randomization of other parts of $\bar{\rnd{\cR}}_t$, only on the positions of $i$ and $j$. Then $\chi_i \geq \chi_j$; from $\alpha(i) > \alpha(j)$ and Assumption~\ref{ass:examination scaling}. Based on this fact, $\Etp{\rnd{z}}$ is bounded from below as
\begin{align*}
  \Etp{\rnd{z}} =
  \chi_i \alpha(i) - \chi_j \alpha(j) \geq
  \chi_i (\alpha(i) - \alpha(j))\,,
\end{align*}
where the inequality is from $\chi_i \geq \chi_j$. Moreover, $\mathbb{P}_{t - 1}(\rnd{z} \neq 0)$ is bounded from above as
\begin{align*}
  \mathbb{P}_{t - 1}(\rnd{z} \neq 0)
  & = \mathbb{P}_{t - 1}(\rnd{z} = 1) + \mathbb{P}_{t - 1}(\rnd{z} = -1) \\
  & \leq \chi_i \alpha(i) + \chi_j \alpha(j) \\
  & \leq \chi_i (\alpha(i) + \alpha(j))\,,
\end{align*}
where the first inequality is from inequalities $\mathbb{P}_{t - 1}(\rnd{z} = 1) \leq \chi_i \alpha(i)$ and $\mathbb{P}_{t - 1}(\rnd{z} = -1) \leq \chi_j \alpha(j)$, and the last inequality is from $\chi_i \geq \chi_j$.

Finally, we chain all above inequalities and get our first claim. The second claim follows from the observation that $\Etp{- \rnd{z} \mid \rnd{z} \neq 0} = - \Etp{\rnd{z} \mid \rnd{z} \neq 0}$.
\end{proof}

\begin{lemma}
\label{lem:concentration} 
Let $S_1 = \set{(i, j) \in [K]^2: i < j}$ and $S_2 = \big\{ (i, j) \in [K]^2 : i > j \big\}$. Let
\begin{align*}
  \cE_{t, 1}
  & = \big\{\forall (i, j) \in S_1: 
   \frac{\alpha(i) - \alpha(j)}{\alpha(i) + \alpha(j)} \rnd{n}_t(i, j) - 2 \sqrt{\rnd{n}_t(i, j) \log(1 / \delta)} \leq
  \rnd{s}_t(i, j)\big\}\,, \\
  \cE_{t, 2}
  & = \set{\forall (i, j) \in S_2: \rnd{s}_t(i, j) \leq 2 \sqrt{\rnd{n}_t(i, j) \log(1 / \delta)}}\,.
\end{align*}
Let $\cE = \bigcap_{t \in [n]} (\cE_{t, 1} \cap \cE_{t, 2})$ and $\ccE$ be the complement of $\cE$. Then $\mathbb{P}(\ccE) \leq \delta^\frac{1}{2} K^2 n$.
\end{lemma}
\begin{proof}
First, we bound $\mathbb{P}(\overline{\cE_{t, 1}})$. Fix $(i, j) \in S_1$, $t \in [n]$, and $(\rnd{n}_\ell(i, j))_{\ell = 1}^t$. Let $\tau(m)$ be the time of observing item pair $(i, j)$ for the $m$-th time, $\tau(m) = \min \set{\ell \in [t]: \rnd{n}_\ell(i, j) = m}$ for $m \in [\rnd{n}_t(i, j)]$. Let $\rnd{z}_\ell = \rnd{c}_\ell(\rnd{\cR}_\ell^{-1}(i)) - \rnd{c}_\ell(\rnd{\cR}_\ell^{-1}(j))$. Since $(\rnd{n}_\ell(i, j))_{\ell = 1}^t$ is fixed, note that $\rnd{z}_\ell \neq 0$ if $\ell = \tau(m)$ for some $m \in [\rnd{n}_t(i, j)]$. Let $\rnd{X}_0 = 0$ and
\begin{align*}
  \rnd{X}_\ell =
  \sum_{\ell' = 1}^\ell \condEsub{\rnd{z}_{\tau(\ell')}}{\rnd{z}_{\tau(\ell')} \neq 0}{\tau(\ell') - 1} - \rnd{s}_{\tau(\ell)}(i, j)
\end{align*}
for $\ell \in [\rnd{n}_t(i, j)]$. Then $(\rnd{X}_\ell)_{\ell = 1}^{\rnd{n}_t(i, j)}$ is a martingale, because
\begin{align*}
  \rnd{X}_\ell - \rnd{X}_{\ell - 1} &=
  \condEsub{\rnd{z}_{\tau(\ell)}}{\rnd{z}_{\tau(\ell)} \neq 0}{\tau(\ell) - 1} -
  (\rnd{s}_{\tau(\ell)}(i, j) - \rnd{s}_{\tau(\ell - 1)}(i, j)) \\
  &=
  \condEsub{\rnd{z}_{\tau(\ell)}}{\rnd{z}_{\tau(\ell)} \neq 0}{\tau(\ell) - 1} -
  \rnd{z}_{\tau(\ell)}\,,
\end{align*}
where the last equality follows from the definition of $\rnd{s}_{\tau(\ell)}(i, j) - \rnd{s}_{\tau(\ell - 1)}(i, j)$. Now we apply the Azuma-Hoeffding inequality and get that
\begin{align*}
  P\left(\rnd{X}_{\rnd{n}_t(i, j)} - \rnd{X}_0 \geq 2 \sqrt{\rnd{n}_t(i, j) \log(1 / \delta)}\right) \leq \delta^\frac{1}{2}\,.
\end{align*}
Moreover, from the definitions of $\rnd{X}_0$ and $\rnd{X}_{\rnd{n}_t(i, j)}$, and by \cref{lem:click loss}, we have that 
\begin{align*}
\delta^\frac{1}{2} \geq {}  & P\left(\rnd{X}_{\rnd{n}_t(i, j)} - \rnd{X}_0 \geq 2 \sqrt{\rnd{n}_t(i, j) \log(1 / \delta)}\right) \\
{}={} & P\left(\sum_{\ell' = 1}^{\rnd{n}_t(i, j)} 
  \condEsub{\rnd{z}_{\tau(\ell')}}{\rnd{z}_{\tau(\ell')} \neq 0}{\tau(\ell') - 1} - \rnd{s}_t(i, j) \geq
  2 \sqrt{\rnd{n}_t(i, j) \log(1 / \delta)}\right) \\
  {}\geq{} & P\left(\frac{\alpha(i) - \alpha(j)}{\alpha(i) + \alpha(j)} \rnd{n}_t(i, j) -
  \rnd{s}_t(i, j) \geq 2 \sqrt{\rnd{n}_t(i, j) \log(1 / \delta)}\right) \\
  {}={} & P\left(\frac{\alpha(i) - \alpha(j)}{\alpha(i) + \alpha(j)} \rnd{n}_t(i, j) -
  2 \sqrt{\rnd{n}_t(i, j) \log(1 / \delta)} \geq \rnd{s}_t(i, j)\right).
\end{align*}
The above inequality holds for any $(\rnd{n}_\ell(i, j))_{\ell = 1}^t$, and therefore also in expectation over $(\rnd{n}_\ell(i, j))_{\ell = 1}^t$. From the definition of $\cE_{t, 1}$ and the union bound, we have $\mathbb{P}(\overline{\cE_{t, 1}}) \leq \frac{1}{2} \delta^\frac{1}{2} K (K - 1)$.

The claim that $\mathbb{P}(\overline{\cE_{t, 2}}) \leq \frac{1}{2} \delta^\frac{1}{2} K (K - 1)$ is proved similarly, except that we use $\condEsub{\rnd{z}_{\tau(\ell)}}{\rnd{z}_{\tau(\ell)} \neq 0}{\tau(\ell) - 1} \leq 0$. From the definition of $\ccE$ and the union bound,
\begin{align*}
  \mathbb{P}(\ccE) \leq
  \sum_{t = 1}^n \mathbb{P}(\overline{\cE_{t, 1}}) + \sum_{t = 1}^n \mathbb{P}(\overline{\cE_{t, 2}}) \leq
  \delta^\frac{1}{2} K^2 n\,.
\end{align*}
This completes our proof.
\end{proof}

\section{RESULTS WITH NDCG}
\label{sec:results with ndcg}

In this section, we report the NDCG of compared algorithms, which measures the quality of displayed lists. 
Since $\cascadeklucb$ fails in the PBM and we focus on learning from all types of click feedback, we leave out $\cascadeklucb$ from this section. 

In the first two experiments, we evaluate algorithms by their regret in \eqref{eq:regret} and safety constraint violation in \eqref{eq:constraint violation}. Neither of these metrics measure the quality of ranked lists directly. In this experiment, we report the per-step NDCG@$5$ of $\bubblerank$, $\batchrank$, $\toprank$, and $\baseline$ (\cref{fig:ndcg}), which directly measures the quality of ranked lists and is widely used in the LTR literature \cite{Jarvelin:2002:CGE:582415.582418,allan2017trec}. Since the \emph{Yandex} dataset does not contain relevance scores for all query-item pairs, we take the attraction probability of the item in its learned click model as a proxy to its relevance score. This substitution is natural since our goal is to rank items in the descending order of their attraction probabilities~\cite{chuklin2015click}. We compute the NDCG@$5$ of a ranked list $\cR$ as
\begin{align*}
\mathit{NDCG}@5(\cR) =
\frac{\mathit{DCG}@5(\cR)}{\mathit{DCG}@5(\cR^\ast)}\,, \quad
\mathit{DCG}@5(\cR) =
\sum_{k = 1}^{5} \frac{\alpha(\cR(k))}{\log_{2}(k + 1)}\,, 
\end{align*}
where $\cR^\ast$ is the optimal list and $\alpha(\cR(k)$ is the attraction probability of the $k$-th item in list $\cR$. This is a standard evaluation metric, and is used in TREC evaluation benchmarks~\cite{allan2017trec}, for instance. It measures the discounted gain over the attraction probabilities of the $5$ highest ranked items in list $\cR$, which is normalized by the DCG@$5$ of $\cR^\ast$. 

In \cref{fig:ndcg}, we observe that  $\baseline$ has good NDCG@$5$ scores in all click models. 
Yet there is still room for improvement.
$\bubblerank$, $\batchrank$, and $\toprank$ have similar NDCG@$5$ scores after $5$ million steps. But $\bubblerank$ starts with NDCG@$5$ close to that of $\baseline$, while $\batchrank$ and $\toprank$ start with lists with very low NDCG@$5$.

These results validate our earlier findings. As in \cref{sec:results with regret}, we observe that $\bubblerank$ converges to the optimal list in hindsight, since its NDCG@$5$ approaches $1$. As in \cref{sec:safety results}, we observe that $\bubblerank$ is safe, since its NDCG@$5$ is never much worse than that of $\baseline$.

\begin{figure*}[!h]
	\includegraphics{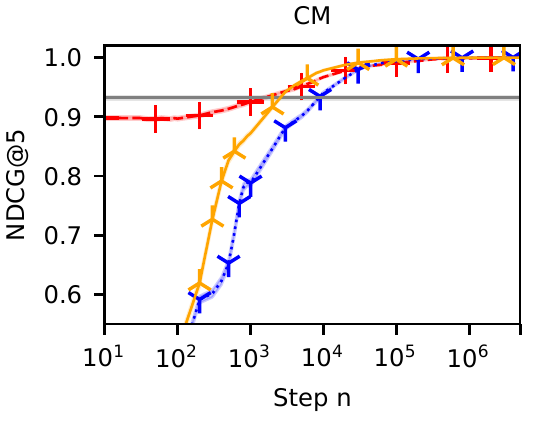}
	\includegraphics{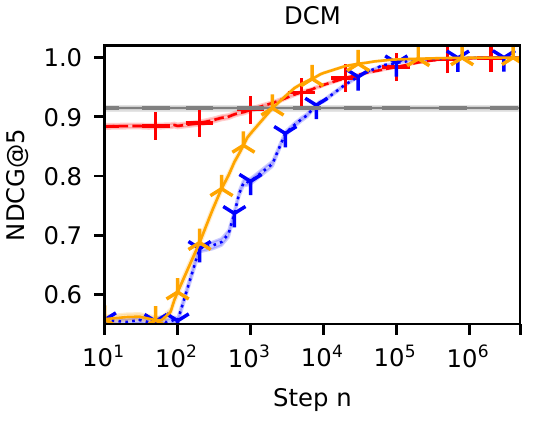}
	\includegraphics{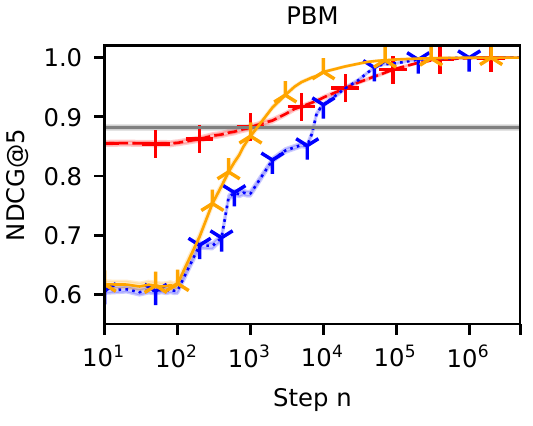}
	\caption{The per-step NDCG@$5$ of $\bubblerank$ (red), $\batchrank$ (blue), $\toprank$ (orange),  and $\baseline$ (grey) in the CM, DCM, and PBM in up to $5$ million steps. Higher is better. The shaded regions represent standard errors of our estimates.}
	\label{fig:ndcg}
\end{figure*}